\documentclass{article} 
\usepackage{ijcai15}

\usepackage{titlesec}
\usepackage{url}
\usepackage{hieroglf}
\newcommand{\hide}[1]{}    
\usepackage{rotate}    
\usepackage{graphicx} 
\usepackage{tabularx}
\usepackage[utf8]{inputenc} 
\usepackage{amsthm} 
\usepackage{amsmath}
\usepackage{amssymb}
\usepackage{textcomp}   
\usepackage{multicol}
\usepackage{enumitem}
\usepackage{bussproofs}
\usepackage{color} 
\usepackage{times} 
\usepackage{tikz}   
\usetikzlibrary{arrows,automata,positioning}
\usepackage[framemethod=TikZ]{mdframed} 
\usepackage{marvosym} 
\usepackage{listings}
\usepackage{caption}    
\usepackage{multirow} 
\usepackage{wrapfig}   
\usepackage[sans]{dsfont}
\usepackage{verbatim}
\usepackage[lofdepth,lotdepth]{subfig}

\newtheorem{counterExample}{Counter example?}
  
\newtheorem{theorem}{Theorem}            
\newcommand{\generate}{\textsf{Gen}}

\newcommand{\update}{\textsf{Update}}
\newcommand{\support}{\ensuremath{\Pi}}

\newcommand{\bquartet}{\textsf{BQuad}}
\newcommand{\cond}{\textsf{Cond}}

\newcommand{\logcon}{\textsf{L}}

\newcommand{\bbase}{\textsf{BBase}}

\newtheorem{definition}{Definition}

\newcommand{\assoc}{\textsf{Assoc}}
 
\newcommand{\lit}{\textsf{Lit}} 
\newcommand{\exclude}{\textsf{Exc}}
\newcommand{\props}{\textsf{Props}}

\newcommand{\andMeta}{\ensuremath{\wedge^{\dagger}}}
\newcommand{\orMeta}{\ensuremath{\vee^{\dagger}}}

\newcommand{\visible}{\textsf{Visible}}

\title{Latent Belief Theory 
    and Belief Dependencies: 
    A Solution to the Recovery Problem 
in the Belief Set Theories}  
\author{Ryuta Arisaka\\ryutaarisaka@gmail.com}

\begin{document}

\maketitle

\begin{abstract}   
    The AGM recovery postulate says: assume a set of propositions $X$; 
  assume that it is consistent and 
  that it is closed under logical consequences; 
  remove 
  a belief $P$ from the set minimally, 
  but make sure that the resultant set is again 
  some set of propositions $X'$ which 
  is closed under the logical 
  consequences; 
  now add $P$ again 
  and 
  close the set under the logical consequences; 
  and we should 
  get a set of propositions 
  that contains 
  all the propositions that were in $X$. 
  This postulate has 
  since met objections; 
  many have observed 
  that it could bear counter-intuitive 
  results. Nevertheless, the attempts
  that have been made so far to amend  
  it either recovered the postulate in full, 
  had to relinquish the assumption of 
  the logical closure altogether, 
  or else had to introduce fresh controversies of 
  their own. We provide a solution to 
  the recovery paradox in this work. 
  Our theoretical basis is the 
  recently proposed belief theory with latent beliefs 
  (simply the latent belief theory 
  for short). 
  Firstly, through an example, 
  we will illustrate that 
  the vanilla latent belief theory 
  can be made more expressive. 
  We will identify that a latent belief, 
  when it becomes visible, may remain visible 
  only while the beliefs 
  that triggered it into 
  the agent's consciousness 
  are in the agent's belief set. 
  In order that such situations 
  can be also handled,  we will enrich 
  the latent belief theory with 
  belief dependencies among attributive beliefs, 
  recording the information as to 
  which belief is supported of 
  its existence by which beliefs. 
  We will show that 
  the enriched latent belief theory does not possess 
  the recovery property. The closure 
  by logical consequences is maintained 
  in the theory, however. Hence 
  it serves as a solution to 
  the open problem in 
  the belief set theories. 
\end{abstract}  
\section{Introduction}     
The belief theory with latent beliefs, 
the latent belief theory for short, 
was recently proposed \cite{Arisaka15perception}. 
In the framework, every evidence $\{P\}^\diamond$ 
is a collection of propositions, 
consisting of one primary proposition 
$P$ and zero or more attributive propositions 
expressed in triples: $P(P_1, P_2)$ for some 
$P_1$ and $P_2$. Each $P(P_1, P_2)$ is basically $P_2$ in any environment that contains $P_1$; otherwise, it is 
a latent belief not presently visible, 
despite its existence, within the environment.
What we have called an environment is, 
in the particular setting of the 
belief theory, a set of propositions and triples 
associated to them. Since they characterise 
the beliefs held by a rational agent, 
an environment is representative of the state of the mind 
of a rational agent's, which 
we may then just call a belief set, as 
in the AGM belief theory \cite{Makinson85}. 
A logical  
closure property holds in the latent belief 
theory: if $P_1, \ldots, P_n$ are in a belief set, 
then any proposition that is a logical consequence 
of any one or any ones in conjunction of 
them is also in the belief set. 
But because the belief sets in the latent  
belief theory 
could also contain those triples,  
they hold 
more information in general than a belief set 
in a traditional belief set theory does. 
To illustrate the point of the triples, suppose that 
$\{P\}^\diamond$ 
consists of $P, P(P_1, P_2)$ and $P(P_3, P_4)$.  
Suppose that an agent believes  
$P_1$, i.e. his/her belief set contains $P_1$. Then $P(P_1, P_2)$ is basically 
$P_2$ to the agent; and $P_2$ is in the belief set. 
But suppose that it does not contain $P_3$; 
then it is not necessarily 
the case that $P_4$ is in the set. 
Suppose that $P_4$ is not in the set, then 
$\{P, P_2\}$ will be 
the agent's perception of 
$\{P\}^\diamond$. Nonetheless, if $P_3$ is 
added to the set, then  
the agent's perception of 
$\{P\}^\diamond$ will be $\{P, P_2, P_4\}$. 
As this example illustrates, the latent 
belief theory captures the dynamic nature 
of a belief/knowledge within the mind of a rational 
agent's. Some constituents of 
$\{P\}^\diamond$ are visible, 
some others may be latent, depending on 
what beliefs are visible to 
his/her conscious mind. \\
\indent Let us contemplate upon the triples. 
In \cite{Arisaka15perception}, 
a latent belief, once it becomes visible 
to an agent, will acquire the equal 
significance in footing to any other visible beliefs 
that he/she holds. In particular, 
if $P(P_1, P_2)$ is latent to him/her, 
and if $P_2$ becomes visible, 
then contraction of  his/her belief set    
by $P$ does not necessarily entail
the loss of $P_2$. There are many 
scenarios that justify the particular behaviour. 
Consider the 
following propositions.\footnote{This 
    example is sketched out of  
Detective Conan.}
\begin{multicols}{2}
    \begin{enumerate}[leftmargin=0.4cm] 
    \item $P_1$: The nerdy-looking 
        boy is Conan.   
    \item $P_2$: There was a high school 
        kid, Shinichi Kudo, who was a renowned detective. 
    \item $P_3$: Conan is Shinichi Kudo.  
\end{enumerate}
\end{multicols}
Suppose the following structure
for $\{P_1\}^\diamond$: 
$\{P_1\}^\diamond = (\{P_1\}, \{P_1(P_2, 
    P_3)\})$, having 
the primary proposition $P_1$ as well as 
one attributive proposition $P_1(P_2, P_3)$.
It is not the case 
that $P_1$ implies $P_2$ or $P_3$. Neither 
is it the case 
$P_2$ or $P_3$ $P_1$. 
Now, suppose an agent who, 
among all the other propositions, believes 
$P_1$, but does not believe either of $P_2$ and $P_3$.
That is, suppose that his/her perception 
of $\{P_1\}^\diamond$ is $\{P_1\}$. 
When he/she learns $P_2$, 
then $P_3$ is triggered 
into his/her mind.
His/her perception 
of $\{P_1\}^\diamond$ is now $\{P_1, P_3\}$. 
Let us say that his/her 
belief set is then contracted by $P_1$. But there is 
no reason that $P_3$ must be also dropped off, 
even though it was attributive to $P_1$ when 
it was latent. He/she does not believe
$\{P_1\}^\diamond$ any more; but 
he/she will still believe $P_3$, 
(or $\{P_3\}^\diamond$ which 
includes $P_3$). 
\subsection{The need for tracking 
    belief dependencies among attributive 
    beliefs} 
However, there are other cases where 
the dependency should carry over. Consider 
the following proposition $\{P_1\}^\diamond$: {\it Belief changes can be characterised 
    in logic.} Suppose 
the following propositions. 
\begin{multicols}{2}
    \begin{enumerate}
        \item $P_1$:    
            Belief changes can be characterised 
            in the AGM belief theory. 
         \item $P_2$: Minimal removal 
             of beliefs is not a random operation. 
         \item $P_3$: Other postulates 
              to the existing AGM postulates 
              characterise belief retention more realistically. 
    \end{enumerate}
\end{multicols} 
Suppose that $\{P_1\}^\diamond 
= (\{P_1\}, \{P_1(P_2, P_3)\})$. 
Suppose some rational agent who knows about 
the basic AGM postulates but who does not know 
about the supplementary postulates.  Suppose that 
the agent perceives $\{P_1\}$ for $\{P_1\}^\diamond$. From some psychology magazine, 
he/she perceives $\{P_2\}^\diamond$. Let 
us say for simplicity that $\{P_2\}^\diamond 
= (\{P_2\}, \emptyset)$. 
Upon his/her accepting it, 
$P_3$, a part of $\{P_1\}^\diamond$ hitherto unknown 
to him/her, comes into his/her consciousness. But in no time, 
some external source convinces him/her that 
the decision on 
what beliefs remain and what beliefs get removed 
after a minimal belief contraction is as unpredictable as 
throwing a die. So he/she 
drops $P_2$. In no time, his/her apprehension grows,  
and he/she becomes dismissive of  
logically representing belief changes. 
The proposition $P_1$ must go. But 
so must $P_3$, since it has existed
on the presumption that 
it be possible to characterise belief changes in logic, 
that is, in the AGM theory as he/she perceives it. 
\subsection{Outlines; on the side theme 
    concerning the recovery property; and other remarks}
To bring the extra expressiveness
that differentiates the two cases 
into 
the latent belief theory, 
in Section 2 we enrich the theory 
by defining belief dependencies 
among attributive beliefs. 
The basic idea is to extend 
the definition of an attributive 
belief to include the fourth parameter: 
$P(P_1, P_2, n)$ so that 
$n$ can determine what beliefs $P_2$ will 
depend on  
once it becomes visible. 
As $P_2$ is made to exist by $P$ and $P_1$, 
the number of the possibilities is four, 
and 
it suffices if $n$ ranges over $\{0,1,2,3\}$. 
Now, once $P_2$ becomes visible 
from $P(P_1, P_2, n)$,  we 
can say that 
if $n = 0$, then $P_2$ is an autonomous 
belief, independent of
$P_1$ and of $P_2$; 
if $n=1$, then it is a dependent belief, independent of
 $P_2$ but dependent on $P_1$; 
 if $n=2$, then it is a dependent belief, independent of 
 $P_1$ but dependent on $P_2$; 
 and if $n=3$,  
 then it exists on the existences of 
 both $P_1$ and $P_2$. We store the information 
 of which belief is dependent 
 on what beliefs in a table  
 comprising pairs of 
 a proposition $P$ and a set of propositions $\Gamma$
on which $P$ is dependent. For instance, 
it may contain $(P, \{P_1\})$. 
Suppose that a belief set contains 
$P, P_1$ among all the other beliefs, and 
that it has this table. 
Then, if the belief set is contracted
such that $P_1$ 
no longer remains in the resulting 
belief set, then it must also happen 
that 
the set do not contain $P$ which - 
so does the table say - cannot exist 
unless $P_1$ is visible. 
Whenever a set of new propositions 
are added to or removed from a belief set, 
the change is reflected upon the table whose 
contents are updated appropriately through a set of update postulates.\\ 
\indent In Section 3, we present
 all the belief change postulates, to complete the development 
 of the enriched latent belief theory. 
 We then show that there is 
 no recovery property in our theory. 
 As a brief reminder, 
 the (AGM) recovery postulate says: 
 assume a set of propositions $X$; 
 assume that it is consistent and that 
 it is closed under logical consequences; 
 remove a belief $P$ from the set minimally, 
 but make sure that the resultant set is
 again some set of propositions $X'$ 
 which is closed under the logical consequences; 
 now add $P$ again and close the set under 
 the logical consequences; and we should 
 get a logically closed set of propositions that contains 
 all the propositions that were in $X$. 
 This postulate has since met objections, 
 many researchers observing that it could 
 bear counter-intuitive results. 
 However, the attempts that have been made 
 so far to amend it either recovered the postulate 
 in full, or else had to introduce fresh 
 controversies of their own \cite{Rott00}. 
 Our theory offers the sought-after solution to the recovery 
 paradox. Section 4 concludes. 
\section{Formalisation of 
    dependencies among  
    attributive beliefs and 
    of belief sets} 
Readers may benefit from 
reading the first few sections of \cite{Arisaka15perception}. 
Although this section is 
technically self-contained, 
a detailed intuition is not given to each 
definition due to space limitation, which is, however, found in the reference. 
For the intuition of the key new notations as well as 
examples that illustrate why they 
are introduced, readers may find it useful 
to refer back to 
Section 1. \\
\indent Let us assume a set of possibly uncountably 
many atomic propositions. We denote 
the set by $\mathcal{P}$, and refer
to each element by $p$ with or without 
a subscript. More general propositions are constructed
from $\mathcal{P}$ and the logical
 connectives of propositional classical logic: 
 $\{\top_0, \bot_0, \neg_1, \wedge_2, \vee_2\}$. 
 The subscripts denote the arity. 
 Although the classical implication $\supset_2$ 
 is not used explicitly, it is derivable 
 from $\neg$ and $\vee$ in the usual manner: 
 $p_1 \supset p_2 \equiv \neg p_1 \vee p_2$. 
The set of literals, i.e. 
any $p$ or $\neg p$ for $p \in \mathcal{P}$, 
is denoted 
by $\lit$. The set of all the propositions 
is denoted 
by $\props$; and each element 
of $\props$ is referred to
by $P$ with or without a subscript. 
Let us assume that, given any $O \subseteq 
2^{\props}$, $\logcon(O)$ is the set of 
all the propositions that are the logical consequences of 
any (pairs of) elements in $O$. A set of 
propositions $U$ is said to be consistent 
iff for any $P \in \props$, if $P \in U$, 
then $\neg P \not\in U$; and if $\neg P \in U$, 
then $P \not\in U$. Let us assume 
that, for any tuples of some sets $(U_1, \ldots, U_k)$, 
we have $\pi_i((U_1, \ldots, U_k)) = U_i$, for $1 \le i \le k$. 
Let us further assume that 
the union of two tuples of sets: $(U_1, \ldots, U_k) 
\cup (U_1', \ldots, U_k')$ is $(U_1 \cup U_1', 
\ldots, 
U_k \cup U_k')$. 
\hide{ 
We define a table 
of propositions and their support sets 
as a family of the ordered set:  
$\widetilde{\props} \times 
2^{\widetilde{\props}}$. 
We denote each such table by $\support$ with or 
without a subscript. Each one of them 
is defined to 
satisfy the following conditions. 
First, for each $P \in \props$, 
it holds that there exists 
some support $\Gamma$ such that 
$\forall P \in \props\ 
\exists \Gamma.(P, \Gamma) \in \support$; 
second, if $(P, \Gamma_1), 
(P, \Gamma_2) \in \support$, then 
$\Gamma_1 = \Gamma_2$.  \\
\indent We define a function $\update$ that takes 
a table, a proposition, and an element of  
$\mathfrak{N}$ and returns a table: 
$\update(\support, P, \Gamma) = (\support \backslash
\support(P)) \cup (P, \pi_2(\support) \cup \Gamma)$.  
}
\begin{definition}[Associations and attributive 
    beliefs\cite{Arisaka15perception}] 
  An association tuple 
  is a tuple $(\mathcal{I}, X, \assoc)$.  
  Let $\mathfrak{N}$ be 
  $\{0, 1, 2, 3\}$. Then 
  $\mathcal{I}$ is a mapping 
  from $\lit$ to $2^{\props \times \props \times 
      \mathfrak{N}}$.
  $X$ is an element of $2^{\props}$. 
  And $\assoc$ is a mapping from 
  $\props$ to $2^{\props \times \props \times
      \mathfrak{N}}$.
  Let $\exclude$ be a mapping from 
  $\props$ to $2^{\props}$ 
  such that $\exclude(P) = 
  \logcon(\{P\}) \cup \{P_1 \in \props \ | \ 
      P \in \logcon(\{P_1\})\}$. 
  Then $\mathcal{I}$ is defined to satisfy
  that, for any $P \in \lit$, 
  if either $P_1 \in \exclude(P)$ 
  or $P_2 \in \exclude(P)$, then 
  $(P_1, P_2) \not\in \mathcal{I}(P)$. 
  $\assoc$ is defined to satisfy 
  (1) that if $P$ is a tautology,
  then $\assoc(P) = (\emptyset, \emptyset, 0)$; 
  (2) that if $\neg P$ is tautology, then $\assoc(P) 
  = (\props, \props, 0)$; and (3) that, if neither; 
  {\small 
      \begin{itemize}[leftmargin=0.4cm]
          \item $\assoc(P) = \mathcal{I}$
              if $P \in \lit$. 
          \item $\assoc(P_1 \wedge P_2) = 
              (\assoc(P_1) \cup \assoc(P_2)) 
              \downarrow \exclude(P_1 \wedge P_2)$ 
              where $(U_1, U_2, n) \downarrow U_3 
              = (U_1', U_2', n)$ such that 
              $U_{1,2}' = U_{1,2} \backslash U_3$. 
          \item $\assoc(P_1 \vee P_2) 
              = \assoc(P_1 \wedge P_2)$ 
              if $P_1, P_2 \in X$. 
          \item $\assoc(P_1 \vee P_2) 
              = \assoc(P_i)$ if 
              $\neg P_j \in X$ for 
              $i,j \in \{1,2\}, i \not= j$.  
          \item 
$\assoc(P_1 \vee P_2) 
        = \assoc(P_i) \downarrow \exclude(P_1 
        \wedge P_2)$ if $P_i \in X$ and 
        $P_j, \neg P_j \not\in X$ for $i,j 
        \in \{1,2\}, i \not= j$. 
          \item $\assoc(P_1 \vee P_2)$ 
              consists of all the pairs 
              $(P_x, P_y, n)$ 
              satisfying the following, otherwise: 
              there exists $(P_a, P_A, n_1)$ in 
              $\assoc(P_1)$ 
              and there exists $(P_b, P_B, n_2)$ 
              in $\assoc(P_2)$ such that 
              (1) $P_x = P_a$; (2) 
              $\logcon(P_a) = \logcon(P_b)$;  
              (3) $n_1 = n_2$; 
              (4) either $P_B \in \logcon(P_A)$ 
              or $P_A \in \logcon(P_B)$;
              (5) if $P_B \in \logcon(P_A)$, 
              then $P_y  = P_B$, else if 
              $P_A \in \logcon(P_B)$, 
              then $P_y = P_A$; and (6) 
              $P_x, P_y \not\in \exclude(P_1 \wedge P_2)$.
          \item $\assoc(\neg (P_1 \wedge P_2)) 
              = \assoc(\neg P_1 \vee \neg P_2)$. 
          \item $\assoc(\neg (P_1 \vee P_2)) 
              = \assoc(\neg P_1 \wedge \neg P_2)$.
 \end{itemize} 
  } 
  We call each $P(P_1, P_2, n)$ for 
  some  beliefs $P, P_1$ and $P_2$ and 
  some $n$ 
  a belief quadruple, 
  and denote the set of belief quadruples
  by $\bquartet$. We define the set 
  $\{P(P_1, P_2, n) \in \bquartet \ | \ 
      [(P_1, P_2, n) \in \assoc(P)] 
      \andMeta [(P_1, P_2) \not= 
      (\props, \props)]\}$\footnote{
      In lengthy formal expressions, 
      we use meta-connectives
      $\andMeta, \orMeta, \rightarrow^{\dagger}, 
      \forall, \exists$ in place 
      for conjunction, disjunction, 
      material implication, universal quantification
      and existential quantification, each 
      following the semantics in classical logic.} 
  to be the set of beliefs attributive 
  to $P$. We denote the set by $\cond(P)$. 
  We denote $\bigcup_{P \in \props'} \cond(P)$ 
  by $\cond(\props')$ where $\props' \subseteq 
  \props$. If $\props' = \props$, 
  we denote it simply by 
  $\cond$. 
\end{definition}   
The third condition of the sixth item
for disjunction, 
which says that 
$(P_a, P_A, n_1)$ in $\assoc(P_1)$ 
and $(P_b, P_B, n_2)$ in $\assoc(P_2)$ 
are not comparable 
unless  
both of the $P_A$ and $P_B$  
have the same attributive belief dependency (i.e. 
$n_1 = n_2$), 
could be possibly relaxed to be less conservative. 
We will leave the consideration to a future work.  
\subsection{Belief sets and axioms, and update postulates} 
In our enriched latent belief theory, 
a belief base is defined to contain 
a subset of $\props$ and  
a set of quadruples: $P_a(P_b, P_c, n)$ 
where 
$P_a, P_b, P_c \in \props$ and 
$n \in \mathfrak{N}$. Additionally, 
it is defined to contain  
a table consisting of 
pairs of 
$(P, \Gamma) \in 
\props \times 2^{\props}$, 
which records 
which belief is dependent on what 
beliefs. Let us denote a set of 
the pairs by $\Pi$ with or without a subscript. 
Let us call 
some tuple $(\Gamma, \Delta, \Pi)$ 
for some $\Gamma \in 2^{\props} \backslash 
\emptyset$, some $\Delta \in 2^{\bquartet}$ 
and some $\Pi$ a belief base. 
We denote the set of belief bases 
by $\bbase$, and 
refer to each element by $B$ with or without 
a subscript.  A belief set is defined to be an element 
of 
$\bbase$ that satisfies the following 
axioms. 
\begin{enumerate}[leftmargin=0.4cm]
    \item $\logcon(\pi_1(B)) = \pi_1(B)$ 
         (Logical closure).   
    \item If $P \in \pi_1(B)$, 
        then there is a finite subset $X$ of 
        $\pi_1(B)$ such that $P \in \mathsf{L}(X)$. 
        (Compactness). 
    \item If $P \not\in \pi_1(B)$, 
        then for any $P_1, P_2 \in \props$ and any 
        $n \in \mathfrak{N}$ it holds that 
        $P(P_1, P_2, n) \not\in \pi_2(B)$
        (Attributive 
        belief adequacy).  
        \hide{ \item For any finite $k$, $(P_1, \Gamma_1), 
        \ldots, (P_k, \Gamma_k) \in 
        \support$ 
        and some non-empty $\Gamma_i$ 
        for each 
        $1 \le i \le k$ iff 
        $(P_a, \Gamma_a) \in \support$ for 
        some non-empty $\Gamma_a$ 
        for each  
        $P_a \in \logcon(\{P_1, \ldots, P_k\})$ 
        (Logical closure: support table).  
    }
    \item If $P \in \pi_1(B)$, 
        then $(P, \Gamma) \in \pi_3(B)$ 
        for some  $\Gamma$ 
        (Support adequacy 1).     
    \item If $P \not\in \pi_1(B)$, 
        then $(P, \Gamma) \not\in \pi_3(B)$ 
        for any $\Gamma$ 
        (Support adequacy 2). 
    \item If $(P, \Gamma) \in \pi_3(B)$, 
        then $\Gamma \not= \emptyset$ (Support 
        sanity). 
    \item If $(P_1, \Gamma_1), (P_2, \Gamma_2), 
        (P_1 \vee P_2, \Gamma)
        \in \pi_3(B)$, 
        then $\logcon(\Gamma_1) \cup \logcon(\Gamma_2) = 
        \logcon(\Gamma)$ (Disjunctive 
        support propagation). 
    \item If $(P_1, \Gamma_1), (P_2, \Gamma_2), 
        (P_1 \wedge P_2, \Gamma) \in 
        \pi_3(B)$, then  
        $\logcon(\Gamma_1) \cap \logcon(\Gamma_2) 
        = \logcon(\Gamma)$ 
        (Conjunctive support propagation). 
    \item If $(P_1, \Gamma_1), (P_2, \Gamma_2) 
        \in \pi_3(B)$ 
        such that 
        $\logcon(P_1) \subseteq \logcon(P_2)$, 
        then $\logcon(\Gamma_2) \subseteq \logcon(\Gamma_1)$
        (Support monotonicity).  
        \hide{ 
    \item If $(P_1 \vee P_2, \Gamma) 
        \in \pi_3(B)$ and 
        if $(P_1, \Gamma_1), (\neg P_2, \Gamma_2) 
        \in \pi_3(B)$, 
        then  
        $\logcon(\Gamma) = \logcon(\Gamma_1)$ 
        (Disjunctive support identity).   
    \item If $(P_1 \wedge P_2, \Gamma), 
        (P_1, \Gamma_1), (P_2, \Gamma_2)
        \in \pi_3(B)$, 
        and if $\logcon(\Gamma_1) = \logcon(\Gamma_2)$, 
        then $\logcon(\Gamma) = \logcon(\Gamma_1)$
        (Conjunctive support identity).  
    }
    \item If $P$ is a tautology such that 
        $P \in \pi_1(B)$, then 
        if $(P, \Gamma) \in \pi_3(B)$, 
        then $\top \in \Gamma$ 
        (Tautological support). 
\end{enumerate}  
Compared to the definition of a belief set 
as found in \cite{Arisaka15perception}, 
this definition does not conduct the fixpoint 
iterations. For the axioms around the third component 
of a belief base (the items 
from 4 to 10), 
insertion of a couple of notes here may 
be useful. 
The (Support adequacy 1), the (Support adequacy 2) and the (Support sanity) ensure that if $B$ is a belief set, 
then that $P$ is in $\pi_1(B)$ 
means that $(P, \Gamma)$ for a non-empty $\Gamma 
\in 2^{\props}$ is in $\pi_3(B)$, and vice versa.  
The (Disjunctive support propagation) 
and the (Conjunctive support propagation) 
say how supporting propositions 
are determined, deterministically up to $\logcon$, 
along the $\logcon$ ladders. 
The basic functionality of the (Support monotonicity) 
is to make $\logcon(\Gamma_1) = \logcon(\Gamma_2)$ 
in case $P_1$ and $P_2$ are indistinguishable 
in $\logcon$. 
Finally the (Tautological support) effectively states that  
a tautological belief in a belief set 
is independent of any non-tautological beliefs. 
We say that a belief base $B$, a belief set in particular, 
is consistent iff  $\pi_1(B)$ is consistent. \\
\indent The following sets of 
postulates: one for when 
an item is added to it, as characterised 
by the operator $\mathsf{\circ}_B$ 
for $B \in \bbase$ which takes 
a belief base and a pair of    
a proposition and a set of propositions
to return a belief base;
and one for when 
an item is removed from it, 
as characterised by the operator $-_B$ 
for $B \in \bbase$ which takes 
a belief base and a set of propositions 
to return a belief base, 
define rules for updating  
the support table (which is the third component 
of a belief base). \\
\textbf{Support table augmentation operator} $\mathsf{\circ}$   
satisfies the following. 
\begin{enumerate}[leftmargin=0.4cm]   
    \item For each {\small $(P, \Gamma) \in 
            \support_1$},  
        if $(P, \Gamma_x) \in \pi_3(B)$, then 
        $(P, \Gamma_x \cup \Gamma \cup \bigcup_{\{ 
            (P_1, \Gamma_1) 
            \in \support_1 \ | \ 
            \logcon(P) \subseteq \logcon(P_1)\}} 
    \Gamma_1) \in \pi_3(B \circ_B \support_1)$.
    \item For each {\small $(P, \Gamma) \in 
            \support_1$},  
        if {\small $(P, \Gamma_x) \not\in \pi_3(B)$}
         for no $\Gamma_x$, then 
        $(P, \Gamma \cup \bigcup_{\{ 
            (P_1, \Gamma_1) 
            \in \support_1 \ | \ 
            \logcon(P) \subseteq \logcon(P_1)\}} 
    \Gamma_1) \in \pi_3(B \circ_B \support_1)$.
\item If $(P_x, \Gamma_x) \in \pi_3(B)$ 
    and if, for all $(P, \Gamma) \in \Pi_1$, 
    it holds that $P_x \not\in \exclude(P)$, 
    then $(P_x, \Gamma_x) \in \pi_3(B 
    \circ_B \Pi_1)$.   
\item If {\small $(P, \Gamma) \not\in \pi_3(B)$}, 
    and if {\small $(P_1, \Gamma_1) \not\in \Pi_1$} 
    for any $\Gamma_1$ and for any  
    $P_1$ such that 
    {\small $\logcon(\{P_1\}) = \logcon(\{P\}) \not= \logcon(\{\top\})$},  
    then {\small ${(P_2, \Gamma_2) \not\in \pi_3(B \circ_B 
            \Pi_1)}$} for 
    any $\Gamma_2$ and any $P_2$ 
    such that {\small $\logcon(P_2) = \logcon(P)$}. 
    \hide{ \item If $(P_1, \Gamma_1) \in \pi_3(B \circ_B \Pi_1)$, 
    then either $(P_1, \Gamma_1) \in \pi_3(B)$, 
    or else $(P_1, \Gamma) \not\in \pi_3(B)$ 
    for any $\Gamma$, or else  
    $(P_1, \Gamma) \in \pi_3(B)$ 
    such that  
    $\Gamma_1 = \Gamma \cup  $ (anyway, to recover).  
}
\item $\pi_i(B) = \pi_i(B \circ_B \Pi_1))$ 
    for $i \in \{1,2\}$. 
\item $B \circ_B (P, \Gamma)$ 
                    satisfies the earlier axioms. 
\end{enumerate}  
To explain these a little,  
the fifth postulate ensures that 
the operation $\circ_B$ acts, if any, 
only upon $\pi_3(B)$, leaving 
$\pi_{1,2}(B)$ intact. Because of this, 
whatever changes that $\circ_B$ operation 
would make to $\pi_3(B)$, the sixth postulate 
guarantees, through the (Support adequacy 1) 
and 
the (Support adequacy 2), that 
any $(P, \Gamma) \in \pi_3(B \circ_B \Pi_1)$ 
is linked to $P \in \pi_1(B \circ_B \Pi_1)$: 
it is in $\pi_3(B \circ_B \Pi_1)$ 
only if $P \in \pi_1(B \circ_B \Pi_1)$. 
Now, as to what those supporting sets of 
beliefs are for each belief, if 
$B$ satisfies the axioms 7 - 10, 
then determining the supporting set of beliefs 
for each of key beliefs
suffices to determine all the other supporting 
sets for each belief, due again 
to the sixth postulate which includes (Disjunctive 
support propagation) and (Conjunctive support propagation). 
As for what the key beliefs are, 
there are those beliefs in $B$ that are unrelated 
to the elements of $\Pi_1$ by $\logcon$. Their supporting set 
of beliefs should not change, which is ensured 
in the third postulate. Apart from those beliefs, 
it suffices to ensure the first two postulates 
to determine all the other supporting sets. 
Finally, the fourth postulate ensures that 
the change to $B$ should be minimal 
by the $\circ_B$ operation. \\
\textbf{Support table reduction operator} 
$\mathsf{-}$ satisfies the following. Here {\small $X \backslash\backslash Y$} denotes
{\small $\{P \in X \ | \ [\logcon(\{P\}) = \logcon(\{\top\})] 
    \orMeta [\neg^{\dagger} 
    \exists P_1 \in Y.\logcon(\{P_1\}) = \logcon(\{P\})]\}$}.  
\begin{enumerate}[leftmargin=0.4cm]  
    \item $[(P_1, \Gamma_1 \backslash \backslash 
          \Gamma) \in 
          \pi_3(B -_B \Gamma)]$ 
          if $[(P_1, \Gamma_1) \in \pi_3(B)]$.     
      \item If $(P_1, \Gamma_1) \in 
          \pi_3(B -_B \Gamma)$, 
          then $(P_1, \Gamma_x) \in 
          \pi_3(B)$ 
          for $\Gamma_x = \Gamma_1 \cup 
          \Gamma$ or else 
          $\Gamma_x = \Gamma_1$. 
     \item $\pi_i(B) = \pi_i(B -_B \Gamma)$ for 
         $i \in \{1,2\}$. 
     \item $B -_B \Gamma$ 
             satisfies the earlier axioms except 
             for (Support sanity). 
\end{enumerate}   
The reduction operation is simpler, 
trying to remove 
matching elements off each $(P, \Gamma) 
\in \pi_3(B)$. 
Then the earlier axioms, in particular 
the (Support adequacy 1) and
the (Support adequacy 2) 
link the elements to the first component 
of $B$. The reason 
that we do not include the (Support sanity) 
in the fourth postulate is just so that 
we can get the belief contraction 
operation right in the next section. 
\section{Belief change postulates} 
Let us define that $\maltese$ is 
the belief expansion operator in our belief 
theory, $\fallingdotseq$ the belief 
contraction operator, 
and $\bigstar$ the belief revision operator. 
We require each one of them 
to be a fixpoint iterating process. 
Since the belief revision operation 
has been always a derivable operation 
in the AGM tradition, we only 
define $\maltese$ and $\fallingdotseq$; 
and, later on, 
will show how the two operations 
are combined into $\bigstar$. 
The graphical representations  
of $B_{\text{init}} \maltese \{P\}^\diamond$ and 
$B_{\text{init}} \fallingdotseq \{P\}^\diamond$, 
assuming that $B_{\text{init}}$ is a belief set, 
are found in Figure \ref{representations}.\footnote{
Although the inner components have not yet been 
formally defined, we believe that  
the display of the visual representations, 
and then detailing the concepts 
used there finely will better explain these operations.} 
The $\visible_B(\{P\}^\diamond)$ comes 
from \cite{Arisaka15perception}. It is 
{\small $P \cup \{P_2 \ | \ [P(P_1, P_2, n) \in 
    \pi_2(\{P\}^\diamond)] \andMeta 
    [P_1 \in \pi_1(B)]\}$}. 
In both of the diagrams,  $X := Y$ 
denotes the assignment of $Y$ to 
$X$. $\Gamma := X; B_0 := Y$ means that 
the assignment operations are taken 
in sequence: the assignment to $\Gamma$, followed
by that to $B_0$. 
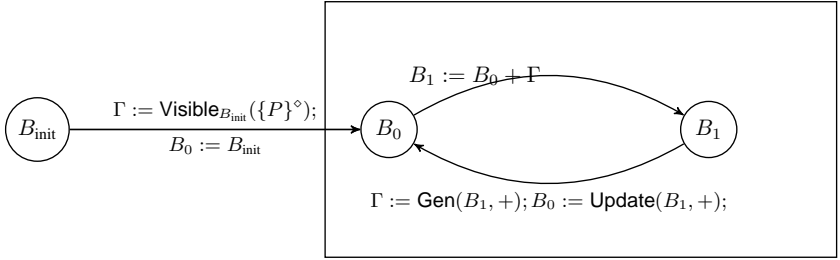
\begin{figure*}[!t]   
    \scalebox{0.85}{ 
    \begin{tikzpicture}[>=stealth',semithick,auto,node distance=5cm] 
                \node[state] (A)   {$B_0$};   
\node[state] (Z) [node distance=5.5cm,left of=A] {$B_{\text{init}}$};
        \node[state] (b) [right of=A] {$B_1$}; 
        \path (A) edge [->,bend left] node [left] {$
            B_1 := B_0 \div 
            \Gamma$} (b)  
        (b) edge [bend left,->] node 
        {{\small $\Gamma := \generate(B_1, \div); 
            B_0 := \update(B_1, \div) 
            $}} (A);    
\draw (-1, 2) rectangle +(8,-4);  
\draw (1,3) node {\huge $B_{\text{init}} \fallingdotseq
    \{P\}^\diamond$ \normalsize 
    {\ }\ 
    (Belief contraction)};  
\path (Z) edge [->] node [above] {{\small $\Gamma :=  
        \visible_{B_{\text{init}}}(\{P\}^\diamond);$}}
(A)  
(Z) edge [->] node [below] {{\small $B_0 := 
        B_{\text{init}}$}} (A); 
    \end{tikzpicture}  
}
    {\ }\newline\newline 
    \scalebox{0.85}{
    \begin{tikzpicture}[>=stealth',semithick,auto,node distance=5cm] 
                \node[state] (A) []   {$B_0$};  
                \node[state] (Z) 
                [node distance=5.5cm,left of=A] {$B_{\text{init}}$}; 
        \node[state] (b) [right of=A] {$B_1$};  
        \path (A) edge [->,bend left] node [left] {$B_1 := B_0 + 
            \Gamma$} (b)  
        (b) edge [bend left,->] node 
        {{\small $\Gamma := \generate(B_1, +); 
            B_0 :=  \update(B_1, +); 
            $}} (A);    
     \draw (-1, 2) rectangle +(8,-4);  
     \draw (1,3) node {\huge $B_{\text{init}} \maltese
         \{P\}^\diamond$ \normalsize 
    {\ }\ 
    (Belief expansion)};  
\path (Z) edge [->] node [above] {{\small $\Gamma :=  
        \visible_{B_{\text{init}}}(\{P\}^\diamond);$}}
(A)  
(Z) edge [->] node [below] {{\small $B_0 := 
        B_{\text{init}}$}} (A); 
        \end{tikzpicture}   
    }
        \caption{Graphical representations 
        of the belief expansion  
        $B_{\text{init}} \maltese \{P\}^\diamond$ 
        and the belief contraction 
        $B_{\text{init}} \fallingdotseq \{P\}^\diamond$ 
        in our latent belief theory.  
        $B_{\text{init}}$ is assumed 
        to be a consistent belief set, although
        the restriction is more a pragmatic 
        than a technical one.} 
    \label{representations} 
\end{figure*} 
Both of the processes continue until the fixpoint
is reached.\\
\indent We now define 
all the participants in the two representations. 
Before moving further, we recall 
(\cite{Arisaka15perception}) 
the postulate for the association tuple. \\
\textbf{Association tuple} has one postulate: 
\begin{enumerate}[leftmargin=0.4cm] 
    \item Each $B$ has 
        an association tuple  
        $(\mathcal{I}, \pi_1(B), \assoc)$ 
        for some $\mathcal{I}$ and 
        some $\assoc$. 
\end{enumerate} 
\textbf{Internal expansion operator} 
$\mathbf{+}$ 
has two postulates: 
\begin{enumerate}[leftmargin=0.4cm] 
    \item $B + \Gamma = (\logcon(\pi_1(B) \cup \Gamma), 
        \cond(\pi_1(B + \Gamma)), \pi_3(B))$.  
    \item If the association tuple 
        for $B$ is 
        $(\mathcal{I}, \pi_1(B), \assoc)$, 
        then that for $B + \Gamma$ 
        is $(\mathcal{I}, \pi_1(B + \Gamma), 
        \assoc)$. 
\end{enumerate} 
\textbf{Internal contraction operator}
$\mathbf{\div}$ 
has the following postulates:  
\begin{enumerate}[leftmargin=0.4cm] 
    \item $B \div \Gamma 
        = (\logcon(B \div \Gamma), \cond(\pi_1(B 
        \div \Gamma)), 
        \pi_3(B))$. 
    \item $\forall P_1 \in \Gamma. 
        P_1 \not\in \logcon(\top) 
        \rightarrow^{\dagger}
        \pi_1(B \div \Gamma)$. 
    \item $\pi_1(B \div \Gamma) \subseteq \pi_1(B)$. 
    \item $(\forall P_1 \in \Gamma. 
        P_1 \not\in \pi_1(B) 
        \orMeta P_1 \in \logcon(\top))
        \rightarrow^{\dagger} B \div \Gamma 
        = B$. 
    \item $[\logcon(\Gamma_1) = 
        \logcon(\Gamma_2)] \rightarrow^{\dagger} 
        [B \div \Gamma_1 = 
        B \div \Gamma_2]$. 
    \item $B \subseteq  
        (B \div \Gamma) + 
        \Gamma$. 
    \item If the association tuple 
        for $B$ is 
        $(\mathcal{I}, \pi_1(B), \assoc)$, 
        then that for $B \div \Gamma$ 
        is $(\mathcal{I}, \pi_1(B \div \Gamma), 
        \assoc)$ (Association update). 
\end{enumerate}  
These postulates closely coincide with 
the AGM postulates \cite{Makinson85}. 
In passing, 
two more postulates may be added 
to the list above: 
{\small $\forall P_1 \wedge P_2 
\in \Gamma.[P_1 \not\in Cn(B) \div \Gamma] 
\rightarrow^{\dagger} [Cn(B) \div \Gamma 
\subseteq Cn(B) \div \Gamma(P_1 \wedge P_2 
\mapsto P_1)]$}; and  
{\small $\forall P_1 \wedge P_2 
    \in \Gamma.(Cn(B) \div 
    \Gamma(P_1 \wedge P_2 \mapsto P_1)) 
    \cap (Cn(B) \div \Gamma(P_1 \wedge P_2 \mapsto 
    P_2)) \subseteq 
    Cn(B) \div \Gamma$}.  
    $\Gamma(P_1 \wedge P_2 \mapsto P_x)$ 
    means to replace all the occurrences of
    $P_1 \wedge P_2 \in \Gamma$ 
    with $P_x$. 
    The two postulates are intended to regulate belief retention \cite{Makinson85}. 
    We are hardly concerned with these 
    supplementary postulates 
    in this particular work, but a mentioning 
    of them 
    may be useful to a reader who 
    is interested 
    in retention of the beliefs based on 
    the concept of
    the epistemic entrenchment. Somehow 
    related to it, recall (\cite{Makinson85}) 
    that, generally speaking 
    neither $B \div \Gamma$ nor $B + \Gamma$ 
    is a deterministic operation. \\
    \textbf{{\generate} and {\update}} are 
    defined as follows.\\
    \indent {\small $\generate(B, 
    \div) :=  
    \{P \ | \ [(P, \Gamma) \in \pi_3(B)] 
        \andMeta  
        ([\forall P_1 \in \Gamma.P_1 \not\in 
        \pi_1(B)] 
        \orMeta [\Gamma = \emptyset])\}$}. {\it Explanation}:  
       $\div$ may remove beliefs off 
       the first component of a belief base. 
       Suppose some beliefs are indeed dropped off, 
       and that we have $B_1 \subset B_0$ 
        (see 
       the graphical representation). 
       Now, it could be that some proposition 
       $P_x \in B_1$ may have 
       lost all the propositions for 
       supporting its existence. 
       Then $P_x$ can no longer subsist in 
       $B_1$, which will 
       be further contracted by all such $P_x$ in the 
       next round of the fixpoint iteration. 
       \\
\indent {\small $\generate(B, +) := 
    \{P_2 \ | \ [P(P_1, P_2, n) \in \pi_2(B)] 
        \andMeta [P, P_1 \in \pi_1(B)]\}$}. 
{\it Explanation}: When a belief base is augmented 
   with a new set of beliefs, 
   it could happen that latent beliefs 
   become visible, which is 
   a subset of all the propositions 
   generated by this set construction. \\
\indent {\small $\update(B, \div) := 
     B -_B  \{P \ | \ [(P, \Gamma) 
         \in \pi_3(B)] \andMeta 
         [P \not\in \pi_1(B)]\}$}. 
     {\it Explanation}:  
       The support table is updated 
      to reflect the loss of beliefs 
      by $\div$. Specifically, any 
      element in $\pi_3(B)$ is removed 
      if the first component of the element 
      is no longer 
      in $\pi_1(B)$. 
      However, recall that the $-$ 
      operation does not 
      satisfy the (Support sanity) axiom. 
      Hence even if the operation 
      should generate some $(P, \Gamma)$ 
      such that $\Gamma = \emptyset$, 
      it is not removed from the third component.\\
\indent {\small $\update(B,+) :=  
    B \circ_B 
    \{(P_2, \Gamma) \ | \ 
        [P(P_1, P_2, n) \in \pi_2(B)] \andMeta  
        [\Gamma = \rho(P(P_1, P_2, n)]\}$} 
where $\rho(P(P_1, P_2, n))$ is $\{\top\}$ if $n = 0$; 
is $\{P\}$ if $n = 1$; 
is $\{P_1\}$ if $n = 2$; 
and is $\{P, P_1\}$ if $n = 3$. \\
There is certain difficulty in 
having an intuitively appealing representation theorem 
of 
$\fallingdotseq$ non-iteratively. 
In the AGM belief contraction operation, 
the belief set as a set of propositions, say $X$, is contracted by some proposition $P$. 
No matter how many 
candidates are for the result of the contraction 
operation,
the candidates are determined 
by $X$ and $P$ alone with no other non-deterministic 
factors. However, in our 
characterisation, the $(k+1)$-th fixpoint iteration  
depends upon the result of the $k$-th 
internal contraction by $\div$, which 
can be known only non-deterministically. 
This makes it hard 
to generate a set-based representation 
of the contraction operation such that  
it retain the same intuitive appeal as the AGM 
representation theorem does. 
For this reason, we regard   
the transition systems shown earlier  
as the representation theorem equivalents for $\fallingdotseq$ and $\maltese$. 
By contrast, the internal operations by $+$ and $\div$ 
almost exactly emulate the AGM expansion and 
contraction operators (Cf. \cite{Makinson85}),  
and the set-based representation is feasible 
for each of them without costing intuitive appeal. 
Particularly 
for $\div$, it goes as follows \cite{Makinson85}. 
Suppose that $B$ is a belief set. Then, 
$\pi_1(B) \div \Gamma = 
\bigcap (\gamma(\Xi(\pi_1(B), \Gamma)))$. 
$\Xi$ is a mapping from belief sets 
and propositions into belief sets. 
For any belief set $B$ and any $\Gamma$,  
we say that a belief set $B_1$ satisfying 
$\pi_1(B_1) \subseteq \pi_1(B)$ 
is a maximal subset of $\pi_1(B)$ for $\Gamma$ 
iff 
\begin{enumerate}[leftmargin=0.4cm]
    \item For any $P_1 \in \Gamma$, 
        $P_1 \not\in \pi_1(B_1)$ if 
        $P_1$ is not a tautology. 
    \item For any belief set $B_2$, 
        if $\pi_1(B_1) \subset 
        \pi_1(B_2) \subseteq \pi_1(B)$, 
        then there exists some 
        $P_a \in \Gamma$ such that 
        $P_a \in \pi_1(B_2)$. 
\end{enumerate} 
We define $\Xi(\pi_1(B), \Gamma)$ 
to be the set of all the subsets 
of $\pi_1(B)$ maximal for $\Gamma$. We further 
define a function $\gamma$, so that, 
if $\Xi(\pi_1(B),\Gamma)$ is not empty, 
then $\gamma(\Xi(\pi_1(B), \Gamma))$ 
is a subset of $\Xi(\pi_1(B), \Gamma)$; 
or if it is empty, it is simply $\pi_1(B)$. 
Then we have that $\pi_1(B) \div \Gamma 
= \bigcap(\gamma(\Xi(\pi_1(B), \Gamma)))$. \\
\indent From $\maltese$ and 
$\fallingdotseq$, 
we define the 
belief revision operator:\\
{\small ${B\ \bigstar \ 
        \{P\}^\diamond}
= {(B \ \fallingdotseq \ 
    (\visible_B^-(\{P\}^\diamond)), \emptyset)}  
\ \maltese \\ 
{\ }\qquad\qquad\qquad(\visible_B(\{P\}^\diamond), \cond(\visible_B(\{P\}^\diamond))$}. \\  
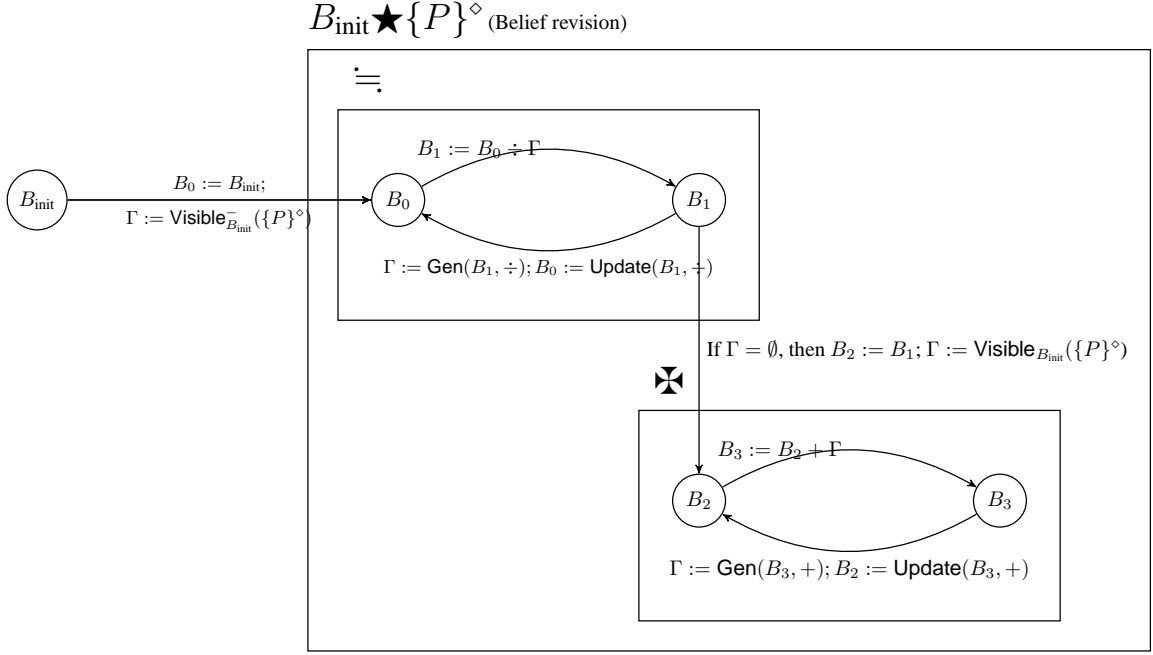
\begin{figure*}[!t]    
\label{revisionOP}     
    \scalebox{0.8}{ 
    \begin{tikzpicture}[>=stealth',semithick,auto,node distance=5cm,scale=0.5]  
        \node[state] (Z) [node distance=6cm,left of=A]  
         {$B_{\text{init}}$};
        \node[state] (A) {$B_0$};  
        \node[state] (b) [right of=A] {$B_1$}; 
        \node[state] (c) [node distance=5cm,below of=b] {$B_2$}; 
        \node[state] (d) [right of=c] {$B_3$}; 
        \path (Z) edge [->] node[above]  
{\small  $B_0 := 
            B_{\text{init}}; $} (A)  
        (Z) edge [->] node[below]
        {\small $\Gamma := 
            \visible_{B_{\text{init}}}^-(\{P\}^\diamond)$} (A) 
        (A) edge [bend left,->] node [left] {$B_1 := B_0 \div 
            \Gamma$} (b)  
        (b) edge [bend left,->] node 
        {{\small $\Gamma := \generate(B_1, \div); 
            B_0 := \update(B_1, \div) 
            $}} (A)
    (b) edge [->] node {If $\Gamma = \emptyset$, then  
        $B_2 := B_1$; $\Gamma := \visible_{B_{\text{init}}}(\{P\}^\diamond$)} (c) 
    (c) edge [->,bend left] node 
    [left] {$B_3 := B_2 + \Gamma$} (d)  
    (d) edge [->,bend left] node 
    {$\Gamma := \generate(B_3, +); B_2 := \update(B_3, 
        +)$} (c);      
    \draw (-2, 3) rectangle +(14, -7); 
    \draw (8, -7) rectangle +(14, -7);
    \draw (-1, 4) node {\huge $\fallingdotseq$};  
    \draw (9,-6) node {\huge $\maltese$};   
    \draw (2.3, 6) node {{\huge $B_{\text{init}}
            \bigstar \{P\}^\diamond$} (Belief revision)}; 
    \draw (-3, 5) rectangle +(28, -20);
    \end{tikzpicture}   
}
            \caption{The belief revision 
        operation $B_{\text{init}}\ \bigstar 
        \ \{P\}^\diamond$  as a composition 
        of the belief contraction and expansion.} 
\end{figure*} 
\indent And we have the representation in the transition system, 
as shown in Figure 2. For the same reason that has 
hindered
us from having a set-based representation 
of $\fallingdotseq$ without losing 
an appeal to intuition, it is difficult 
to come up with non-iterative postulates 
for $\bigstar$. This is not the sign 
that our theory is not robust. It just confirms
the point that every belief 
change operation is an iterative process in our 
theory. 
\begin{theorem}[Preservation] 
   Let $B$ be a consistent belief set.  
   Then $B\ \maltese \ 
   \{P\}^\diamond$,
    $B\ \fallingdotseq \ 
    \{P\}^\diamond$ and 
    $B\ \bigstar \ 
    \{P\}^\diamond$ 
    are again a belief set. 
\end{theorem}  
\noindent Let us also note the following important result. 
\begin{theorem}[No recovery]  
   Let $B$ be some belief set, 
   and let $\{P\}^\diamond$ be 
   some external information. 
   Then we can find a pair 
   of $B, \{P\}^\diamond$ such that 
   both of the following fail to hold. 
   ${\pi_1(B) \subseteq \pi_1((B \fallingdotseq \{P\}^\diamond) 
       \maltese \{P\}^\diamond)}$. 
       ${\pi_1(B) \subseteq \pi_1((B \fallingdotseq 
           \{P\}^\diamond)}  
       \maltese (\visible_B(\{P\}^\diamond), \emptyset))$.  
\end{theorem} 
\begin{proof} 
  The first weaker result 
  holds already in the vanilla latent 
  belief theory; Cf. \cite{Arisaka15perception}.  
  To give the evidence that 
  $\pi_1(B) \not\subseteq 
  \pi_1((B \fallingdotseq 
  \{P\}^\diamond) \maltese 
  (\visible_B(\{P\}^\diamond, \emptyset)))$, 
  suppose that our language  
  is constructed from $p_1, p_2, p_3$ 
  and the logical connectives 
  $\{\top, \bot, \wedge, \vee\}$. 
  Suppose that no pairs of  the three propositions 
  are associated by the logical consequences. 
  Now, suppose that $B$ is a belief set; $\pi_1(B) = 
  \logcon(\{p_1, p_3\})$; 
  $\pi_2(B)$ contains $p_1(p_2, p_3, 1)$ but does not 
  contain any $P_x(P_y, P_z, n)$ such that 
  $\logcon(P_z) \subseteq \logcon(p_3)$ 
  or that $\logcon(p_3) \subseteq \logcon(P_z)$; and that 
  $\pi_3(B)$ contains 
  $(p_3, \{p_1\})$, among others. 
  Then $p_3 \not\in (B \fallingdotseq \{p_1\}^\diamond)$. Suppose that 
  $B \fallingdotseq \{p_1\}^\diamond$ 
  does not retain any logical consequence 
  of $p_3$ apart from tautological propositions. 
  Then the belief set contains 
  no propositions $P_{\alpha}$ such that 
  $\logcon(P_{\alpha}) = \logcon(P_z \supset p_3) \not= 
  \logcon(\top)$. 
  Then $p_3 \not\in (B \fallingdotseq \{p_1\}^\diamond) 
  \maltese (\{p_1\}, \emptyset)$, as required. 
\end{proof}    
\noindent This result can be also adapted 
to the AGM belief theory, 
even though there are no 
quadruples in the AGM setting, so long as 
the same dependencies among 
propositions
are facilitated 
in the AGM operations of belief expansion 
and contraction. With this remark, we are positive that 
this work has truly offered a satisfactory solution to the  
recovery paradox as far as the cases similar to 
the scenarios 
in the opening examples are concerned. 
\section{Conclusion}   
We have presented an enriched latent belief theory, 
in which belief dependencies based on 
how a latent belief has become visible 
to agents' belief sets can be expressed. We have 
shown that there is no recovery 
postulate in our belief theory, thus 
giving a solution to the 
recovery paradox that has been in the belief set theories. 
Our theory indicates that the time 
may have come to study beyond the 
G{\"{a}}rdenfors' principle \cite{Gardenfors90}: 
``If a belief state is revised by a sentence $A$,
then all sentences in $K$ that are independent 
of the validity of $A$ should be retained 
in the revised state of belief", upon which 
many current works on 
dependencies, such as \cite{Oveisi15} 
for a recent one, appear to be based. The examples in Section 1 in any case 
indicate that beliefs, even though not 
connected by logical consequences, can still be
related by their contents, which is the standpoint 
that has been already taken in \cite{Arisaka15perception}.
Although we had to present just the result, 
there appears to be some observation that has not been 
detailed in the literature around 
the recovery paradox. In another work 
of ours, which is going to be much less technical, 
we will have an overview of the paradox: 
what it was, and how it was naturally 
resolved. One future work 
will focus on applications 
of this enriched latent belief theory.  
Connection to 
theories of concurrency in 
computer science will be sought after.  
\section*{Related works}  
Several works questioned the AGM recovery postulate 
\cite{Hansson91,Makinson87,Rott00}. The attempt 
to amend it has not been successful within 
the belief set setting, however. 
\hide{ We, incidentally, do not 
necessarily uphold the view that 
there is a real problem in the AGM recovery postulate 
in itself, nor do we voice disagreement over the ground 
on which it arose. But this matter will be discussed 
in detail in the mentioned supplementary work. }
There are works on belief revision with non-classical 
logics that do not have the classical logical consequence
relations. Not surprisingly, the recovery 
property does not necessarily hold in the setting, 
as evidenced also in the belief base theory, 
which was cultivated notably by Hansson and others. 
\bibliographystyle{plain} 
\bibliography{references} 
\end{document}